\documentclass[conference]{IEEEtran}
\usepackage{microtype}
\usepackage{graphicx}
\usepackage{subfigure}
\usepackage{booktabs} 
\usepackage{hyperref}

\usepackage{amsmath}
\usepackage{amssymb}
\usepackage{footnote}
\usepackage[normalem]{ulem}
\makesavenoteenv{tabular}
\makesavenoteenv{table}
\usepackage{xcolor}

\def\X{\mathbf X}
\def\inv{^{-1}}

\def\H{\mathbf H}
\def\x{\mathbf x}
\def\xs{\underline {\mathbf x}}
\def\Xs{\underline { \mathbf X} }
\def\r{\mathbf r}
\def\y{\mathbf y}

\def\emax{\mathbf{e}_\mathrm{max}}
\def\emin{\mathbf{e}_\mathrm{min}}

\def\bbeta{\boldsymbol \beta}

\def\intinf{\int_\infty^\infty}

\def\t{^\top}
\def\eps{\boldsymbol \varepsilon}

\def\E{\mathbb E}
\def\V{\mathbb V}
\def\N{\mathcal N}

\DeclareMathOperator*{\argmax}{argmax~} 
\DeclareMathOperator*{\argmin}{argmin~} 
\DeclareMathOperator{\sign}{sign} 

\def \real{\rm I\!R}

\newtheorem{theorem}{Theorem}
\newtheorem{corollary}{Corollary}

\newtheorem{proof}{Proof}

\def\BibTeX{{\rm B\kern-.05em{\sc i\kern-.025em b}\kern-.08em
    T\kern-.1667em\lower.7ex\hbox{E}\kern-.125emX}}

\begin{document}

\title{Active Learning for High-Dimensional Binary Features
}

\author{\IEEEauthorblockN{Ali Vahdat}
\IEEEauthorblockA{\textit{Huawei Noah's Ark Lab} \\
Montreal, Canada \\
ali.vahdat@huawei.com}
\and
\IEEEauthorblockN{Mouloud Belbahri}
\IEEEauthorblockA{\textit{Huawei Noah's Ark Lab} \\
Montreal, Canada \\
mouloud.belbahri@huawei.com}
\and
\IEEEauthorblockN{Vahid Partovi Nia} 
\IEEEauthorblockA{\textit{Huawei Noah's Ark Lab} \\
Montreal, Canada \\
vahid.partovinia@huawei.com}
}

\maketitle

\begin{abstract}
Erbium-doped fiber amplifier (EDFA) is an optical amplifier/repeater device used to boost the intensity of optical signals being carried through a fiber optic communication system. A highly accurate EDFA model is important because of its crucial role in optical network management and optimization. The input channels of an EDFA device are treated as either on or off, hence the input features are binary. Labeled training data is very expensive to collect for EDFA devices, therefore we devise an active learning strategy suitable for binary variables to overcome this issue. We propose to take advantage of sparse linear models to simplify the predictive model. This approach simultaneously improves prediction and accelerates active learning query generation. We show the performance of our proposed active learning strategies on simulated data and real EDFA data.
\end{abstract}
\begin{IEEEkeywords}
Active Learning, EDFA, Exponential Family, Binary Features, BIC
\end{IEEEkeywords}

\section{Introduction}
\label{sect:intro}

We start by introducing the Erbium-doped fiber amplifier (EDFA) device, and subsequently review some of the works in the literature on active learning.

\subsection{EDFA}
\label{sect:edfa_background}

The EDFA equipment is an optical repeater/amplifier device to boost the intensity of optical signals through optical fiber. A highly accurate EDFA model is critical for a number of different reasons, such as: i) to improve performance of light path setup, ii) to calculate optical signal to noise ratio (OSNR), and iii) to predict the path performance. However, collecting labeled data from EDFA devices is expensive, involving an expert technician in lab environment to play with the device and collect and record the input-output signal levels. This is where active learning (AL) strategies is used to collect more data to improve the accuracy of the EDFA model.

For an EDFA equipment the input signal is received at a \textit{channel}'s input and the amplified signal leaves the same channel's output. A typical EDFA device supports between 40 and 128 channels depending on the manufacturer and type. Each channel can carry the optical signal for a different service, but not all channels carry service signal at all times. Channels carrying service signals are interpreted as \emph{on} and others are interpreted as dummy or \emph{off} channels. Therefore input signals can be deemed as binary or $x=\{-1, 1\}$. Rather than the actual strength of the channel output, we are interested in the channel gain which is $$y_c = \mathrm{gain}_c = \mathrm{output}_c - \mathrm{input}_c$$ where $c$ is the channel index, $c \in \{1, \ldots, C\}$, ($C$ is the number of channels for a given EDFA device). Therefore, $y$ is a continuous random variable. 

Here, our objective is to use active learning to improve performance of a simple model for a single EDFA channel. Channel outputs are independent given channel inputs, so the generalization towards multivariate output is straightforward.

\subsection{Active Learning}
\label{sect:al_background}

State-of-the-art machine learning (ML) algorithms require an unprecedented amount of data to learn a useful model. Although there is access to a huge amount of data, most of the available data are unlabeled, and labeling them are often time-consuming and/or expensive. This gives rise to a category of ML algorithms that identify the most promising data subset to improve model performance. A data point selected to be inquired about its label is usually referred to as a \emph{query}, and the entity providing the label for the queried data point is usually called an \emph{oracle}. Oracle could be a human, a database, or a software providing the label for the query.

ML algorithms are capable of achieving better performance if the learning algorithm is involved in the process of selecting the data points it is trained on. This is the main objective of AL methods. AL-based methods usually achieve this enhanced performance by selecting the data points they deem more useful for training, based on some form of i) uncertainty measure i.e. using the data points where the ML algorithm is most uncertain about or ii) some form of data representativeness, i.e. using the data points that are good representatives of the data distribution, see \cite{settles2012activeBook} for details.

Depending on the type of data, there are two main variations of AL algorithms; stream-based and pool-based. In stream-based AL the learning algorithm, e.g. a classifier, has access to each unlabeled data point sequentially for a short period of time. The AL algorithm determines whether to request a query or discard the request \cite{cohn1996active}. In pool-based AL \cite{lewis1994sequential}, the learning algorithm e.g. a classifier, has access to the pool of all unlabeled data. At each iteration the algorithm queries the label of an unlabeled data point from the oracle. Our proposed method falls within this category, where most AL research has been focused. Common AL algorithms improve a classifier algorithm, devised for data with continuous features and discrete response. Motivated with the EDFA application, we develop an AL algorithm for data with discrete features and a continuous response.

Methods using \emph{uncertainty sampling} \cite{sharma2013most, ramirez2014anytime} query data points with the highest uncertainty. After observing the a new point in the uncertain region, the learning algorithm becomes more confident about the neighboring subspace of the queried data point. The query strategy maintains the exploration-exploitation trade-off \cite{osugi2005balancing}. In a classification task entropy is used as the uncertainty measure. However, motivated from support vector machines, some authors define uncertainty through the decision boundary  \cite{tong2001support, baram2004online}. For regression tasks prediction variance is the common uncertainty measure. Methods based on universal approximators, such as neural networks, lack analytical form for prediction variance, so empirical variance of the prediction is used instead. 

Methods that focus on a single criteria to select a query often limit the active learning performance. AL algorithms are often trapped due to the sampling bias. Therefore an exploration-exploitation method with a large proportion of random sampling during the early queries is adopted. Some authors also consider combining different criteria \cite{cebron2009active, bondu2010exploration} or selecting the strategies adaptively for a better performance. \cite{baram2004online, hsu2015active, chu2016activeTransfer} perform adaptive strategy selection by connecting the selection problem to multi-arm bandit methods. \cite{baram2004online} uses unlabeled data points as arms (slot machines), whereas \cite{hsu2015active} uses AL strategies as arms in the bandit problem. 

In \cite{konyushkova2017learning} authors train a regressor that predicts the expected error reduction for a candidate data point in a given learning state. The experience from previous AL outcomes is utilized to learn strategies for query selection. \cite{ali2014active} proposes to train multiple models along with the active learning process. They construct two sets simultaneously; a biased training set that improves the accuracy of individual models, and an unbiased validation set that helps to select the best model. \cite{sabharwal2016incremental} automatically selects a model, tunes its hyperparameters, evaluates models on a small set of data, and gradually expands the set if the model is promising.

Section~\ref{sect:background} lays out the modeling and query generation. Section~\ref{sect:experimental} provides results on simulated data, which gives us the insight we need to apply our active learning strategies to the real data.

\section{Methodology}
\label{sect:background}

A typical pool-based AL algorithm has access to a small pool of labeled data\footnote{Note that we show univariate variables with lowercase letters, e.g. $y$, vectors with bold lowercase letters, e.g. $\x$, and matrices with bold uppercase letters, e.g. $\X$.}, 
$$D_L=\{(\x_1, y_1), (\x_2, y_2), \ldots, (\x_n, y_n)\},$$ where $\x_i\in \real^p$ is the predictor and $y_i \in \real$ is the response. Also, there is a potentially larger pool of unlabeled data $$D_U=\{\x_{n+1}, \x_{n+2}, \ldots, \x_m\}.$$ 

In the EDFA application $\x$ is the set of $p$ input channels, and $y$ is one of the output channels selected for modeling. Output channels are conditionally independent, which allows to model each output channel independently.

A typical AL algorithm starts by training a model using the labeled pool $D_L$. Then, at each iteration an AL strategy select a promising data point $\x_i$ from the unlabeled pool $D_U$ and queries its label $y_i$. Once label is retrieved for $\x_i$, this data point is removed from $D_U$ and $(\x_i, y_i)$ is added to $D_L$. The classifier now is trained on the new pool $D_L$, including the recently added $(\x_i, y_i)$. This process is repeated until a termination criteria -- usually a sampling budget $T$ -- is reached. With a small sampling budget $T$, the goal of AL is to find the best sequence of data points to be queried in order to maximize the average test accuracy of the model.

\label{sect:linear}

Suppose the response variable observations come from a distribution in the exponential family with canonical link. Its probability density function is defined as

\begin{align}
    f(y_i \mid \eta_i, \phi) = \mathrm{exp} \Big( \frac{y_i \eta_i - b(\eta_i)}{a_i(\phi)} + c(y_i, \phi) \Big).
    \label{eq:exp_fam}
\end{align}

Here, $\eta_i$ and $\phi$ are location and scale parameters. The functions $a_i(.), b(.)$ and $c(.)$ are known. 

Motivated from generalized linear models, one may introduce a link function $g$ and focus on modelling 
$$ \eta_i = g(\mu_i) = \x_i^{\top} \bbeta,$$

where $\mu_i = \mathbb{E}(y_i)$ is the dependent variable's mean, $\x_i$ is a $p$-dimensional vector of predictors and $\bbeta$ is the $p$-dimensional vector of coefficients. It can be shown that if $y_i$ has a distribution in the exponential family, then

\begin{align*}
    &\mathbb{E}(y_i) = \mu_i = \frac{\partial b(\eta_i)}{\partial \eta_i}, \\
    &\mathbb{V}(y_i) = \sigma_i^2 = \frac{\partial^2 b(\eta_i)}{\partial^2 \eta_i} a_i(\phi).
\end{align*}

For an EDFA equipment, $y$ is a continuous random variable. Therefore, one can model the relationship between $y$ and $\x$ with a Gaussian distribution. The Gaussian distribution with mean $\mu_i$ and variance $\sigma^2$ is part of the exponential family with a linear link function $g$ such as
$$ \eta_i = g(\mu_i) = \mu_i = \x_i^{\top} \bbeta,$$
and $b(\eta_i) = \frac{1}{2} \eta_i^2$, $a_i(\phi) = \phi$, and $\phi = \sigma^2$. In this case, the generalized linear model falls into the linear regression context.


There is a strong reason to start with a linear model. A linear model with interactions fully describes any complicated model built over discrete features, suitable for the EDFA data setting. 


The coefficients $\bbeta$ are unknown in practice, and are estimated using least squares  
$$\hat \y = \X \hat \bbeta, \textrm{~~where~~} \hat\bbeta = (\X\t\X)\inv \X\t\y,$$
where $\y_{n \times 1}$ is the vector of observed response, $\X_{n \times p}$ is row-wise stacked matrix of predictors. Therefore, 
$$\hat\y = \X(\X\t\X)\inv \X\t \y = \H\y,$$
where $\H$ is the projection matrix.

In ultra high-dimensional settings ($p \gg n$) where most feature selection methods fail computationally, it is suggested to order the features $\x$ with a simple measure of dependence like Pearson correlation and select some of relevant features. This simplifies the ultra high-dimensional setting to a high-dimensional setting \cite{FanLv_SURE_2008} where $p \sim n$ and feature selection methods are computationally feasible. In AL, ultimately, a query is generated with an estimated model dimension $m\ll p$.

\subsection{Feature ordering}

In active learning for EDFA, model building starts with small number of observations $n$, say $n\approx 20$. If the feature dimension $p\gg 40$, least squares estimate of coefficients $\hat\bbeta$ are ill conditioned, because $\X\t\X$ is rank-deficient. Regularization, feature selection, dimension reduction, are common methods to resolve this problem. Here we focus on sparse estimation of the coefficients often implemented by $L_1$ regularization. Sparse estimation selects only a small subset of features to predict the response. Active learning requires to re-estimate the model after each new observation is added, and feature selection significantly accelerates frequent model updates.

However, $L_1$ regularization is still computationally challenging for large $p\gg n$. \cite{FanLv_SURE_2008} recommends \emph{sure} screening to pre-select a subset of features with a large absolute correlation (with the response), and then to run $L_1$ regularization on this subset. They show this subset selection keeps important features with a high probability. 

Therefore, the $L_1$ regularization is run over \emph{sure} pre-selected features, to reduce the dimensionality from order of $p\gg n$ to $p\sim n$. This dimension reduction is fast and requires only $O(np)$ operations to compute the correlations, and $O(n\log n)$ to order them. The total computation complexity of \emph{sure} screening is $O(p n \log n)$.

Once pre-selected features are chosen, an $L_1$ regularization method is used to choose the number of features in the model. The $L_1$ regularized regression \emph{lasso} (least absolute shrinkage and selection operator) by \cite{Tibshirani_lasso_1996} solves
{\small
\begin{eqnarray}\label{eq:lasso}
\ell(\bbeta\mid \lambda) &=& \frac{1}{2}(\y-\X\bbeta)\t(\y-\X\bbeta) + \lambda \sum_{j=1}^p |\beta_j|, \\  && \lambda\in \real, \bbeta\in\real^p.\nonumber \\
\hat{\bbeta}\mid \lambda &=& \argmin_{\bbeta} \ell(\bbeta\mid\lambda)\nonumber
\end{eqnarray}
}
Setting the regularization constant $\lambda=0$ returns the least squares estimates which performs no shrinking and no selection. If $\X\t\X$ is diagonal, the \emph{lasso} reduces to soft-thresholding, a common computationally fast estimation method in compressed sensing.

For a given $\lambda>0$ the regression coefficients $\hat\bbeta$ are shrunk towards zero, and some of them are set to zero (sparse selection) similar to soft-thresholding. However, the fitting algorithm is more challenging as $\X\t\X$ is not diagonal, which is the case of EDFA data. \cite{friedman2010regularization} proposed a fast coordinate descent method to fit \eqref{eq:lasso} for a given $\lambda>0$. In practice an appropriate value of $\lambda$ is unknown, and cross-validation is used  over a grid to search for a convenient regularization constant. Choosing appropriate $\lambda$ using cross-validation does not provide sparse consistent models. Even does not guarantee estimation consistency, and moreover, is computationally expensive. 

\cite{shao_penalization_1996} showed sparsity and parameter consistency do not coincide for $L_1$ regularized regression such as the \emph{lasso}. Two approaches are suggested to address this issue; i) estimate the model dimension $m<p$ consistently, and use the estimated model dimension to re-estimate the regression parameter $\bbeta$, or ii) use a non-convex regularization such as the \emph{scad} of \cite{FanLi_SCAD_2001}. Here we use the first approach and estimate the model dimension contently, and then refit the model with non-zero parameters to recover regression parameter consistency. We avoid cross-validation because it is i) inconsistent, ii) computationally challenging.  Instead we derive the predictive distribution, also called the \emph{evidence} which is known to be sparse consistent \cite{shao_penalization_1996}. We show the predictive distribution of regression with certain Gaussian prior mimics the Bayesian information criterion (BIC). The BIC of \cite{Schwarz_BIC_1978} is derived under asymptotic approximation, but our predictive distribution is also valid for small sample sizes, suitable for AL setting and EDFA data.  

The \emph{lar} (least angle regression) algorithm \cite{Efronetal_lar_2004} computes the \emph{lasso} with some minor modifications,  but its implementation is a lot faster, specially if $\lambda$ is unknown. However, even \emph{lar} for  $p\gg n$ is slow. This is why we recommend to pre-select using \emph{sure} screening and feed the selected features to \emph{lar} algorithm. With reasonable dimension $p\approx n$ \emph{lar} method is fast. The \emph{lar} algorithm efficiently computes the path of $\hat\bbeta(\lambda_j)$ over a sequence of $\lambda_j$ that the parameter dimension changes. The \emph{lar} algorithm finds the path of $\lambda_j$ and individual estimates $\hat\bbeta\mid\lambda_j, j=1,\ldots, p$, with the same computational complexity of a single least square.

\subsection{Feature selection}

In a linear model with $p$ covariates, there are $2^p$ candidate models. 
Choosing the model dimension and choosing one of the $\lambda_j$'s are inter-related. The choice of model dimension is an integer value $m\in \{1,\ldots,p\}$. The length of sequence of $\lambda_j \approx \min (n, p)$. So one can choose a value $\lambda_j$, and evaluate the model for the effective dimension imposed by that $\lambda_j$. Repeating the same process for all model dimensions and picking the best model dimension $m$ from the $p$ candidate models is wise, we escape from evaluating $2^p$ candidate models and reduce it to only $p$ model evaluation. This approach is well-known as \emph{two-stage} selection, which guarantees the statistical consistency of model dimension, and also the statistical consistency of the estimated parameters simultaneously. 

Select a value $\lambda_j$ and use its corresponding nonzero $\hat\bbeta(\lambda_j)$ to create a new design matrix $\X_j$ with dimension $n\times m$. The best model is chosen by maximizing the predictive log likelihood $\ell_j$, i.e., the best model dimension is  

$$m=\argmax_{j} \ell_j, \quad j \in \{1,\ldots, p\}.$$

Theorem \ref{theo:bic} derives the predictive log likelihood for small sample sizes inline with the BIC of \cite{Schwarz_BIC_1978}. It is not difficult to see this predictive model is asymptotically equivalent to the BIC. However, in small samples they behave differently. 

\begin{theorem}\label{theo:bic}
Let $\hat\bbeta$ be the maximum likelihood estimate of $\bbeta$ and $\mathcal{F}$ be the exponential family distribution. Suppose that the observed information matrix $J(\hat\bbeta)$ is positive definite and 
\begin{eqnarray*}
\y & \sim & \mathcal{F},\\
\bbeta & \sim & \N(\hat{\bbeta} , n \{J(\hat\bbeta)\}^{-1}).
\end{eqnarray*}
The predictive log likelihood  
$$\ell_j = \log \intinf \cdots \intinf f(\y \mid \bbeta , \X) d\pi(\bbeta\mid\X)$$
simplifies to 
\begin{eqnarray*}
\ell_j &=& \ell(\hat\bbeta) - \frac{m}{2} \log(n+1) + o(1).\end{eqnarray*}
\end{theorem}

\begin{proof}
Define $\ell(\bbeta)$ as the log-likelihood function of $\bbeta$ given $\y$ and $\X$. Using the second-order Taylor expansion at the maximum likelihood estimate $\hat \bbeta$, we have

\begin{align*}
\begin{split}
 \ell(\bbeta) &= \frac{1}{0!} \ell(\hat \bbeta) + \frac{1}{1!} \frac{\partial \ell(\bbeta)}{\partial \bbeta}  \mid _{\bbeta = \hat \bbeta} (\bbeta - \hat \bbeta)  + \\
& \frac{1}{2!} (\bbeta - \hat \bbeta)^\top \lbrace \frac{\partial^2 \ell(\bbeta)}{\partial \bbeta \partial \bbeta^\top} \rbrace \mid _{\bbeta = \hat \bbeta} (\bbeta - \hat \bbeta) + \mathcal{O}_p(||\bbeta - \hat \bbeta||^3),
\end{split}
\end{align*}

or equivalently,

\begin{align*}
\begin{split}
\ell(\bbeta) = \ell(\hat \bbeta) - \frac{1}{2} (\bbeta - \hat \bbeta)^\top \lbrace J(\hat \bbeta) \rbrace (\bbeta - \hat \bbeta) + o(1),
\end{split}
\end{align*}

Hence, by subtracting in the likelihood of $\mathcal{F}$, we have 
\begin{align}
\begin{split}
 f(\y \mid \bbeta, \X) &= \mathrm{exp}\{ \ell(\hat \bbeta)\} \times \\ 
& \exp \lbrace - \frac{1}{2} (\bbeta - \hat \bbeta)^\top \lbrace J(\hat \bbeta) \rbrace (\bbeta - \hat \bbeta) \rbrace + o(1).
\end{split}
\label{eq:ylike1}
\end{align}

Also, the prior distribution of $\bbeta$ is
\begin{align}
\begin{split}
 \pi(\bbeta \mid \X) &= \mid 2 \pi n \{J(\hat \bbeta)\}^{-1} \mid^{-\frac{1}{2}} \times \\
& \mathrm{exp} \lbrace -\frac{1}{2n} (\bbeta - \hat \bbeta)^\top \lbrace \{J(\hat \bbeta)\}^{-1} \rbrace^{-1} (\bbeta - \hat \bbeta)\rbrace.
\end{split}
\label{eq:betalike1}
\end{align}

Therefore, the product of \eqref{eq:ylike1} and \eqref{eq:betalike1} is then given by

\begin{align*}
\begin{split}
& f(\y \mid \bbeta, \X) \pi(\bbeta \mid \X) = \\
&\mathrm{exp} \lbrace \ell(\hat \bbeta) \rbrace \mid 2 \pi n \{J(\hat \bbeta)\}^{-1} \mid^{-\frac{1}{2}} \times\\
& \mathrm{exp} \lbrace -\frac{(n+1)}{2n} (\bbeta - \hat\bbeta)^\top \lbrace \{J(\hat \bbeta)\}^{-1} \rbrace^{-1} (\bbeta - \hat\bbeta)\rbrace + o(1).
\end{split}
\end{align*}

Now, by taking the integral with respect to $\bbeta$, the predictive likelihood simplifies to

\begin{align*}
\begin{split}
L_j &= \mathrm{exp} \lbrace \ell(\hat \bbeta) \rbrace \mid 2 \pi n \{J(\hat \bbeta)\}^{-1} \mid^{-\frac{1}{2}} \mid 2 \pi \frac{n}{n+1} \{J(\hat \bbeta)\}^{-1} \mid^{\frac{1}{2}} \\ &+ o(1)
\end{split}
\end{align*}

or equivalently,

$$L_j = (n+1)^{-\frac{m}{2}} ~ \mathrm{exp} \lbrace \ell(\hat \bbeta) \rbrace + o(1). $$

Finally, the predictive log likelihood is given by

$$ \ell_j =  \ell(\hat \bbeta) -\frac{m}{2} \mathrm{log}(n+1) + o(1).~ \square$$ 

\end{proof}

\begin{corollary}
Suppose that $$ y_i \sim \mathcal{N}(\x_i^{\top}\bbeta, \sigma^2).$$
The predictive log likelihood simplifies to 

\begin{equation} \label{eq:bic}
\ell_j = \mathrm{constant} - \frac{1}{2 \sigma^2 }(\y-\hat\y)\t(\y-\hat\y) - \frac{m}{2} \log(n+1)
\end{equation}

where $\hat \y$ is the predicted response under dimension $m$
$$\hat\y = \X_j (\X_j\t \X_j)\inv \X_j \y.$$ 
\end{corollary}

In case of $\y$ having a Gaussian distribution, the Taylor expansion is exact since $ \mathcal{O}_p(||\bbeta - \hat \bbeta||^3) = 0$. The positive constant $2\sigma^2$ is unknown in practice, and does not play a role in maximization.  

\subsection{Linear query generation}
\label{sect:linearquery}

Here we focus on linear models for query generation. Linear models are attractive because the class of  linear models including main effects with interactions cover any complex function on discrete features. We start with a linear model with main effects only (and no interaction) to create an extremely fast query generation, called \emph{query-by-sign}. Then generalize it to a linear model with main effects and pair-wise interactions to trade-off some computation for better accuracy. We call this method \emph{query-by-variance}. However, the model may include higher order significant interactions. We address this issue by using bagged trees to produce \emph{query-by-bagging}.

In AL context, the objective is to request a new observation that most improves the model performance. There are two major paradigms to interpret model performance; i) smaller variance of prediction $\hat\y$, and ii) smaller variance of estimators $\hat\bbeta$. Here we take the former approach as it makes more sense for the EDFA application, and focus on improving prediction accuracy as the objective.

In an AL setting, at each iteration, a new data point $\xs_{1 \times p}$ is requested, and after observing its response variable $y(\xs)$ the training set is updated. Therefore, we use the notation $\hat\bbeta(\xs)$ to emphasize that this new $\hat\bbeta$ is estimated after adding this new observation to the previously observed design matrix $\X_{n \times p}$. 

The new design matrix, after adding the new observation is
$$\Xs_{(n+1)\times p}=\left[
\begin{array}{c}
\X_{n\times p} \\ 
\xs_{1\times p}
\end{array}
 \right],$$

Note that $\X$ is the training data already observed, and variance is a function of the new observation $\xs$ only. To improve prediction accuracy, we query a new observation $\xs$ under which the model prediction has the largest uncertainty $\V\{\hat y(\xs)\}$. From \emph{conditional variance theorem} \cite{sheldon2002first} 
\begin{align*}
\V\{\hat y(\xs)\}= \E_{\xs}[\V\{\hat y(\xs)\mid \xs \}]  + \V_{\xs}[ \E \{\hat y (\xs) \mid \xs\}].
\end{align*}
Since $\xs$ is a query and under our control, this simplifies to
$$\V\{\hat y(\xs)\}= \V\{\hat y(\xs)\mid \xs \}.$$

From linear model assumption the response variance  $\V\{ y(\xs)\mid \xs \}$ is constant $\sigma^2$. Note that the response variance is different from the prediction variance, i.e. $$\V\{ y(\xs)\mid \xs \}\neq \V\{ \hat y(\xs)\mid \xs \}.$$ The response variance $\V\{ y(\xs)\mid \xs \}$ is estimated through the residual mean squares 
$$ \frac{1}{n-p} \r\t\r,$$ 
where $n-p$ is the error degrees of freedom. The residual (or the training error) is $\r=\hat\y-\y$. The prediction variance $\V\{ \hat y(\xs)\mid \xs \}$ requires more elaboration and depends on $\X$ and $\xs$. 

To maximize the prediction variance $\V\{ \hat y(\x)\mid \x \}$ one needs to keep the maximizer scale-invariant, otherwise any direction $\x$ with a large scale $c$ is solution because $\V(c\x)=c^2\V(\x)$.  
Suppose ${\xs}$ is of a fixed norm to avoid scaling, therefore 
\begin{equation}
\argmax\limits_{\xs} \V\{{\xs}\t\hat\bbeta\}=\argmax\limits_{\xs} \sigma^2 {\xs}\t(\X\t\X)\inv \xs,
\label{eq:maxvar}
\end{equation}
where $\sigma^2$ is a constant and can be ignored in maximization.

Equation \eqref{eq:maxvar} is key to active learning for linear models. Suppose the model dimension is estimated properly $m<p$, and $\xs$ is continuous $\xs\in\real^m$. The scale-invariant solution query generation requires maximizing \eqref{eq:maxvar} subject to a bounded norm $\xs\t\xs =c^2$. It is easy to see that such a maximizer has a closed form.

\begin{theorem}\label{theo:eigen}
Suppose $\X\t\X$ is positive definite, therefore $\hat{\xs}=c~ \emin$, where $\emin$ is the eigenvector associated with the smallest eigenvalue of $\X\t\X$ is the solution to the following optimization problem 
\begin{eqnarray}
 & \argmax\limits_{\xs} {\xs}\t(\X\t\X)\inv \xs, \label{eq:eigensol}  \\
& \mathrm{s.t.~~}  \xs\t\xs=c^2. \label{eq:eigenconst} \nonumber 
\end{eqnarray}
\end{theorem}

\begin{proof}
Suppose $\X\t\X$ is positive definite and let $\y = c \inv \xs = (c \inv \xs_1, \ldots, c \inv \xs_{\hat p})$. Therefore, the constraint $\xs^\top \xs= c^2$ is equivalent to $\y^\top \y = 1$. The optimization problem reduces to 
$$ \argmax_{\y} {\y}\t(\X\t\X)\inv \y ~~ \mathrm{s.t.} ~~ \y\t\y = 1,$$

or equivalently for $\y \neq \mathbf{0}$, to

$$ \argmax_{\y} \frac{{\y}\t(\X\t\X)\inv \y}{{\y}\t\y}.$$

Let $\mathbf{Q}$ be the orthogonal matrix whose columns are the eigenvectors of $\mathbf{A} = (\X^\top \X)^{-1}$ and $\mathbf{D}(\lambda)$ the associated eigenvalues diagonal matrix. Suppose that the eigenvalues are ordered such as $\lambda_1 \geq \lambda_2 \geq \ldots \geq \lambda_{\hat p}$. Let $\mathbf{A}^{1/2} = \mathbf{Q} \mathbf{D}(\lambda)^{1/2} \mathbf{Q}^\top$ and $\mathbf{z} = \mathbf{Q}^\top \y$. Therefore, for $\mathbf{z} \neq \mathbf{0}$, 

\begin{align}
\begin{split}
& \frac{{\y}\t(\X\t\X)\inv \y}{{\y}\t\y} = \frac{{\y}\t \mathbf{A}^{1/2} \mathbf{A}^{1/2} \y}{{\y}\t \mathbf{Q} \mathbf{Q}^\top \y} = \frac{\mathbf{z}^\top \mathbf{D}(\lambda) \mathbf{z}}{\mathbf{z}^\top \mathbf{z}} \\
& = \frac{\sum_{j=1}^{\hat p} \lambda_j z_j^2}{\sum_{j=1}^{\hat p} z_j^2} \leq \lambda_1 \frac{\sum_{j=1}^{\hat p} z_j^2}{\sum_{j=1}^{\hat p} z_j^2} = \lambda_1 .
\label{eq:cond1eigen}
\end{split}
\end{align}

Now, for $\y = \emax$, the eigenvector associated with $\lambda_1$, the largest eigenvalue of $(\X\t\X) \inv$, we have 

$$ \mathbf{z} = \mathbf{Q}^\top \emax = (1, 0, \ldots, 0)^\top, $$

because $ \mathbf{e}_j^\top \emax = 1$ for $j=1$ and $0$ otherwise. Hence for this choice of $\y$, we have

\begin{align}
\frac{\emax\t(\X\t\X)\inv \emax}{\emax\t \emax} = \frac{\mathbf{z}^\top \mathbf{D}(\lambda) \mathbf{z}}{\mathbf{z}^\top \mathbf{z}} = \frac{\lambda_1}{1} = \lambda_1.
\label{eq:cond2eigen}
\end{align}

By \eqref{eq:cond1eigen} and \eqref{eq:cond2eigen}, it is straightforward to see that $\hat \y = \emax$ and since $\y = c \inv \xs$, we have $\hat \xs = c ~ \emax$. Consequently $\hat \xs = c ~ \emin$ where this time $\emin$ is the eigenvector associated with the smallest eigenvalue of $\X\t\X$. $\square$
\end{proof}

The computational cost of this solution is $O(m^2)$, which is quite fast for small $m$. The application of \eqref{eq:eigensol} is not restricted to continuous feature space. Suppose that the feature space is binary $\xs \in \{-1,+1\}^{m},$ then $\xs\t\xs = m$ and a relaxed approximate solution is
\begin{equation}
\hat{\xs} = \sign(\emin).
\end{equation}
The linear model is sparse, so the dimension of $\xs$ is negligible ($m \ll p$). The brute force maximizes the objective by trying all $2^{m}$ possible values, so it is combinatorially large. However, for small $m \leq 18 $, exhaustive search is computationally feasible.

\subsection{Ensemble-based query generation}
\label{sect:ensemble}

In many applications the prediction function is a nonlinear function. While a linear model helps to identify important features, they are  not accurate for prediction purpose. As a consequence, an inaccurate prediction model leads to generating suboptimal queries. The model may contain even more than the second order interactions, or the linear model variance assumption $\V(y(\x)) = \sigma^2$ might be wrong. 
In this section we address both issues by fitting a flexible ensemble tree on the sparse features, and relax assumption the constant variance assumption by computing the empirical variance of the prediction. Among many variants of ensemble methods we propose bagging, because empirical  estimation of the variance $\V(y(\x))$ is straightforward. 

Bagging \cite{breiman1996bagging} is a method for fitting an ensemble of learning algorithms trained on bootstrap replicates of the data in order to get an aggregated predictor. Suppose that $B$ bootstrap replicates are sampled from the observed $n$ independent data 
$$(\x_i, y_i), \quad i=1,\ldots, n,$$ 
and for $b = 1, \ldots, B,$ a regression tree $T_b$ is fitted. Therefore the response prediction is
$$\hat \y = \frac{1}{B} \sum_b \hat \y_b, $$
where $\hat \y_b = \hat T_b(\x)$ is the prediction output of a single tree. Hence, the prediction variance $\V\{\hat y(\xs)\mid \xs \}$ is estimated by the empirical variance 
$$ \hat\V\{ \hat y(\xs)\mid \xs \} = \frac{1}{B-1} \sum_b \lbrace \hat y_b(\xs) - \hat y(\xs) \rbrace ^2.$$

In the context of active learning, the query-by-bagging suggests ${\hat \xs}$ that maximizes the empirical variance such as 
\begin{equation}
\hat \xs = \argmax\limits_{\xs} \hat\V\{ \hat y(\xs)\mid \xs \}.
\label{eq:empvar}
\end{equation}

\section{Experimental Analysis}
\label{sect:experimental}

We divide our experimental analysis into two subsections; simulated data, discussed in Section~\ref{sect:simulations}, and the real-world EDFA application, discussed in Section~\ref{sect:application}.

\subsection{Simulations}\label{sect:simulations}

Here we conduct a simulation study to assess the performance of the three proposed active learning methods: \emph{query-by-sign}, \emph{query-by-variance} and \emph{query-by-bagging}. Each method has its associated fitted model; a linear model using main effects only for \emph{query-by-sign}, a linear model using main effects and second order interaction terms for \emph{query-by-variance}, and an ensemble of bagged trees for \emph{query-by-bagging}. We compare the three different query generation strategies against random sampling. We evaluate the performance of these methods by varying the complexity of the simulated data, see Table \ref{tab:queries} for a summary.

\begin{table}[htb!]
\caption{Summary of the proposed query strategies.}
\label{tab:queries}
\centering
{\small
    \centering
    \begin{tabular}{c c c c}
    & Query by & Query by & Query by \\
    & Sign & Variance & Bagging  \\
    \hline
    Modeling & \\
    \hline 
    Used Effects & main &  main + & bagged tress \\
    & & pair-wise & \\
    Ordering & lar & sure, lar & lar \\
    Selection& BIC & BIC & BIC \\
    \hline
    Sampling \\
    \hline
    Used Effects & main & main + & main \\
    & & pair-wise & \\
    Optimization & $\sign(\emin)$ & $\V(\hat y (\xs))$  & $\hat{\V}(\hat{y}(\xs) )$ \\
    \hline
\end{tabular}}

\end{table}

An active learning algorithm in our experiments has the following components: a) A training set $D$, split into labeled ($D_L$) and unlabeled ($D_U$) pools, b) A validation set $V$, c) A sampling budget $T$, and d) Features selection update frequency.

For each scenario, we generate $100$ independent observations for the labeled pool ($D_L$), and $20000$ independent observations for the unlabeled pool ($D_U$). As mentioned earlier, $D_U$ is the pool that responds to the queries by providing the label $y_i$ for the data point $\x_i$. Sampling budget ($T$) is usually set at $1000$, and finally, a set of $2000$ labeled data points is generated for validation purpose ($V$). We use the RMSE for comparing different active learning strategies.


Data is generated with the following specifications:

\begin{enumerate}
\item Draw $(x_1^\star, \ldots, x_p^\star)$ a vector of random variables from a Bernoulli distribution $\mathrm{Ber}(\theta)$,
\item For $j=1, \ldots, p$, let $x_j = 2 x_j^\star - 1$ so that the features space is $\lbrace -1,1 \rbrace^p$,
\item Let $\beta_0$ be a constant and define $(\beta_1, \ldots, \beta_p)$ as a vector of random variables with discrete probability distribution: $\mathrm{Pr} (\beta_j=-3) = \mathrm{Pr} (\beta_j=+3) = 0.5$ for a subset of size $k < p$ of features, and $\mathrm{Pr} (\beta_j=0) = 1$ for the remaining $p-k$ features.
\end{enumerate}

With $p$ features, the most complicated regression model generates a coefficient for the constant $\beta_0$, $p$ coefficients for main effects (features), ${p \choose 2}$ coefficients for second-order (pair-wise) interactions, ..., and one last coefficient for the full-interactions term. In our simulations, with only $k$ non-zero coefficients, ${k \choose 2}$ terms can be included for second-order interactions. Simulated data using this model will guarantee the usefulness of feature selection.

We consider three different scenarios. We assume a linear model $\y=\X \bbeta + \eps$, where the input matrix $\X$ can be composed of either; i) main effects and 2nd order interactions where we highlight the usefulness of query-by-variance  ii) main effects, 2nd and 3rd order interactions where we highlight the usefulness of query-by-bagging and iii) main effects only, where we highlight the usefulness of query-by-sign. 
For the remainder of this section, we fix $k=7$ and $p=40$. We simulate data based on generated models using 2nd and 3rd order interactions. Data generated including 2nd order interactions has an input dimension of $7 + {7 \choose 2} = 28$, and for 3rd order interactions the input dimension is $7 + {7 \choose 2}  + {7 \choose 3} = 63$.

All models are fitted with features selected by the \emph{lar} algorithm. Note that after each set of 50 new observations added to the labeled pool, the feature selection is repeated and the model is updated. The three active learning methods are as follows:

\begin{enumerate}
\item \textbf{query-by-sign}: Fit a linear model using only main effects (pre-selected by constrained \emph{lar}). Query the observation $\hat{\x} = \sign(\emin)$.

\item \textbf{query-by-variance}: Fit a linear model using main effects (pre-selected by constrained \emph{lar}) and the corresponding 2nd order interaction terms. Query the observation $\hat{\x}$ that maximizes the variance of $\hat \y$.

\item \textbf{query-by-bagging}: Fit $10$ bagged regression trees with maximum $10$ features (pre-selected by constrained \emph{lar}) in each tree. Query the observation $\hat{\x}$ that maximizes the empirical variance of $\hat \y$.
\end{enumerate}

In the first scenario data are simulated by a linear model using main effects and 2nd order interactions. Figure \ref{fig:scene1} illustrates the performance of the three different active learning strategies on this data. The query-by-sign (top left) fails because the fitted model only incorporates main effects, and hence is not an accurate approximation of the data. query-by-variance (top right) and query-by-bagging (bottom) active learning strategies outperform the random sampling strategy and eventually find the ``true'' model as the sampling budget increases, however, query-by-bagging finds the ``true'' model more smoothly.

\begin{figure}[tbh]
\centering
\includegraphics[width=0.3\textwidth]{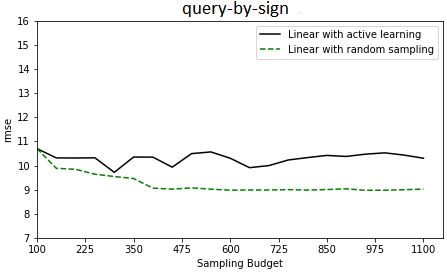}
\includegraphics[width=0.3\textwidth]{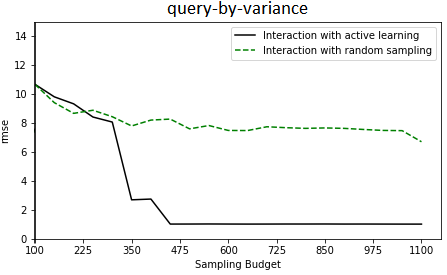}
\includegraphics[width=0.3\textwidth]{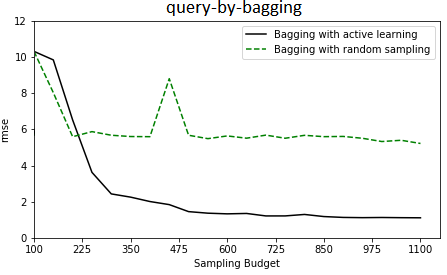}
\caption{Validation RMSE on data simulated by a linear model with main effects and 2nd order interactions using three different AL strategies; query-by-sign (top), query-by-variance (middle), and query-by-bagging (bottom).}
\label{fig:scene1}
\end{figure}

In the second scenario data is simulated by a linear model using main effects, 2nd and 3rd order interactions. Observing the failure of query-by-sign for the less complex data of first scenario, we compare only the query-by-variance to the query-by-bagging methods in this scenario. Figure \ref{fig:scene2} illustrates the results. When the 3rd order interaction terms are added to the simulated model, the query-by-variance (left panel) fails compared to the random sampling strategy. This suggests query-by-variance needs to be adjusted if significant 3rd order interactions are present in the model. However, query-by-bagging (right panel), outperforms the random sampling strategy by a large margin.

\begin{figure}[tbh]
\begin{center}
\includegraphics[width=0.3\textwidth]{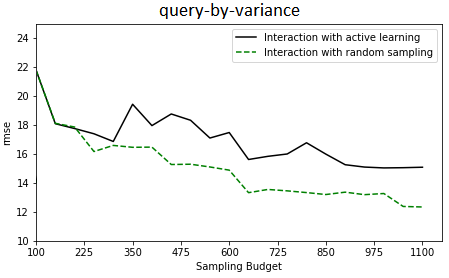}
\includegraphics[width=0.3\textwidth]{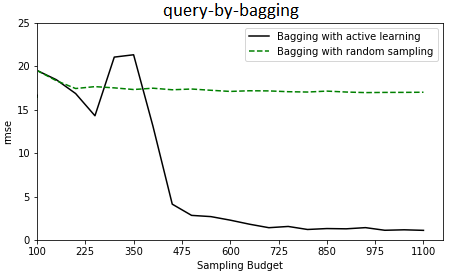}
\end{center}
\caption{Validation RMSE on data simulated by a linear model with main effects, 2nd and 3rd order interactions using two different AL strategies; query-by-variance (top) and query-by-bagging (bottom).}
\label{fig:scene2}
\end{figure}



Table~\ref{tab:runtime} summarizes the run time  for the three  proposed methods. As expected query-by-sign is the strategy with least computational cost and the fastest method. Query-by-variance is more computationally expensive, and query-by-bagging is the most expensive among the three methods. Times reported is the time needed (in seconds) to generate 200 queries with no model update during the query generation. The models are fixed to use 5 features only. Experiments are performed on a laptop with a 2.6 GHz CPU.

\begin{table}[htb!]
\caption{Run time (in seconds) for generation of 200 queries.}
\label{tab:runtime}
\centering
{\small
    \centering
    \begin{tabular}{c c c c}
    & Query by & Query by & Query by \\
    & Sign & Variance & Bagging  \\
    \hline
    Training size & \\
    \hline
    1000 & 2.45 & 5.68 & 6.69 \\
    2000 & 2.54 & 5.57 & 7.19 \\
    3000 & 2.84 & 6.22 & 7.94 \\
    4000 & 3.74 & 7.31 & 9.02 \\
    5000 & 4.09 & 8.13 & 11.01 \\
    \hline
\end{tabular}}

\end{table}

\subsection{Application}
\label{sect:application}

Here We apply our query generation methods to the data collected from the optical amplifier equipment (EDFA).

Our data set contains about $9000$ observations for an EDFA device with $40$ channels. We split the data set into a training set, a validation set, and a test set with $4500$, $2250$, and $2250$ observations, respectively. We further split the training set into a labeled pool of $100$ observations, and an unlabeled pool of $4400$ observations. Sampling budget is $1000$. 




There is a trade-off between maximum number of features to include in the bagging ensemble and the update frequency of feature selection. Using a large number of features in the model renders query generation computationally expensive, and therefore requires a less frequent feature selection update. By keeping the maximum number of features in the model small, we can generate queries faster, and update the features selection more often. For example, if only 18 features are used for bagging and query generation, and features selection is performed every 10 iterations, we can achieve performance observed in left side of Figure \ref{fig:edfa2}. Note that the final validation RMSE has dropped to $0.085$ which is enough to save multiple hours of engineers' time for collecting labeled data. On the right side of this figure we can further observe the increasing model size as more and more observations are queried by the active learning. Although the model can add more useful features or drop less useful ones at each model update step (as can be seen from the oscillating model size graph), the model using active learning strategy takes more advantage of this freedom compared to the random sampling strategy, and reaches the maximum number of features allowed to index for modeling and query generation (i.e. $18$.) The performance of AL strategy increases as model size upper bound increases to 20 or higher, but this comes with a computational cost.

\begin{figure}[tbh]
\begin{center}
\includegraphics[width=0.3\textwidth]{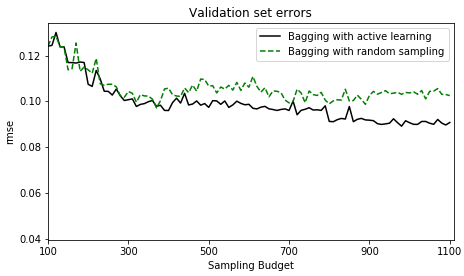}
\includegraphics[width=0.31\textwidth]{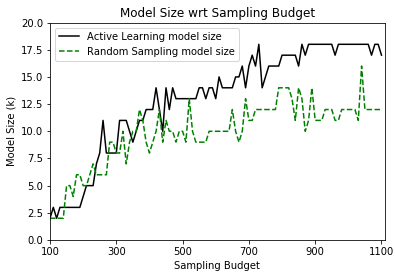}
\end{center}
\caption{Validation RMSE on EDFA data using query-by-bagging (top). The estimated model size as the number of samples increases (bottom).}
\label{fig:edfa2}
\end{figure}

\section{Conclusion}
\label{sect:conclusion}

Active learning helps to make better use of limited labeling budget by integrating data selection process into the learning algorithm. We proposed three different active learning strategies with different computational costs and running time requirements. The simplest strategy, query-by-sign, only considers main effects of a linear model for query generation. Query-by-variance takes advantage of second-order interactions, and query-by-bagging considers high-order interactions by using an ensemble of trees to model data and generate the queries. We simulated data using models with second or third order interactions, and compared the three different active learning strategies. We then applied our findings to EDFA data, a very small and highly complex data set. We observed that query-by-bagging, when tuned properly, improves the model prediction performance and saves engineers' data collection time. Also, the simpler sampling strategy, query-by-variance, displays interesting results, but on data sets with less main effect interactions.


\bibliography{main}

\end{document}